\documentclass[nohyperref]{article}

\usepackage{microtype}
\usepackage{graphicx}
\usepackage{subfigure}
\usepackage{booktabs} 
\usepackage[utf8]{inputenc}
\usepackage[textsize=tiny]{todonotes}
\usepackage{lipsum}
\usepackage{algorithmic}
\usepackage{amsmath,amsfonts,amssymb}
\usepackage{amsthm}
\usepackage{bm}
\usepackage{enumitem}

\usepackage{letltxmacro}
\newlength{\commentindent}
\makeatletter
\renewcommand{\algorithmiccomment}[1]{\unskip\hfill\makebox[\commentindent][r]{$\triangleright$~#1}\par}
\LetLtxMacro{\oldalgorithmic}{\algorithmic}
\renewcommand{\algorithmic}[1][0]{%
  \oldalgorithmic[#1]%
  \renewcommand{\ALC@com}[1]{%
    \ifnum\pdfstrcmp{##1}{default}=0\else\algorithmiccomment{##1}\fi}%
}
\makeatother

\newtheorem{theorem}{Theorem}
\newtheorem{lemma}{Lemma}
\newtheorem{definition}{Definition}

\makeatletter
\renewcommand*\env@matrix[1][c]{\hskip -\arraycolsep
  \let\@ifnextchar\new@ifnextchar
  \array{*\c@MaxMatrixCols #1}}
\makeatother

\allowdisplaybreaks

\usepackage[hyphens]{url}
\usepackage{hyperref}
\usepackage[arxiv]{icml2022}

\usepackage[capitalize,noabbrev]{cleveref}

\icmltitlerunning{Gradients without Backpropagation}

\DeclareMathOperator{\Var}{Var}
\renewcommand{\Re}{\mathbb{R}}

\begin{document}

\twocolumn[
\icmltitle{Gradients without Backpropagation}



\icmlsetsymbol{equal}{*}

\begin{icmlauthorlist}
\icmlauthor{Atılım Güneş Baydin}{cs}
\icmlauthor{Barak A. Pearlmutter}{nuim}
\icmlauthor{Don Syme}{ms}
\icmlauthor{Frank Wood}{ubc}
\icmlauthor{Philip Torr}{eng}
\end{icmlauthorlist}

\icmlaffiliation{eng}{Department of Engineering Science, University of Oxford}
\icmlaffiliation{cs}{Department of Computer Science, University of Oxford}
\icmlaffiliation{ms}{Microsoft}
\icmlaffiliation{ubc}{Computer Science Department, University of British Columbia}
\icmlaffiliation{nuim}{Department of Computer Science, National University of Ireland Maynooth}

\icmlcorrespondingauthor{Atılım Güneş Baydin}{gunes@robots.ox.ac.uk}

\icmlkeywords{gradient descent, backpropagation, forward mode, automatic differentiation}

\vskip 0.3in
]



\printAffiliationsAndNotice{}  

\begin{abstract}
Using backpropagation to compute gradients of objective functions for optimization has remained a mainstay of machine learning. Backpropagation, or reverse-mode differentiation, is a special case within the general family of automatic differentiation algorithms that also includes the forward mode. We present a method to compute gradients based solely on the directional derivative that one can compute exactly and efficiently via the forward mode. We call this formulation the \textbf{forward gradient}, an unbiased estimate of the gradient that can be evaluated in a single forward run of the function, entirely eliminating the need for backpropagation in gradient descent. We demonstrate forward gradient descent in a range of problems, showing substantial savings in computation and enabling training up to twice as fast in some cases.
\end{abstract}

\section{Introduction}
\label{submission}

Backpropagation \citep{linnainmaa1970representation,rumelhart1985learning} and gradient-based optimization have been the core algorithms underlying many recent successes in machine learning (ML) \cite{Goodfellow-et-al-2016,deisenroth2020mathematics}. It is generally accepted that one of the factors contributing to the recent pace of advance in ML has been the ease with which differentiable ML code can be implemented via well engineered libraries such as PyTorch \citep{paszke2019pytorch} or TensorFlow \citep{abadi2016tensorflow} with automatic differentiation (AD) capabilities \cite{griewank2008evaluating,baydin2018ad}. These frameworks provide the computational infrastructure on which our field is built.

Until recently, all major software frameworks for ML have been built around the reverse mode of AD, a technique to evaluate derivatives of numeric code using a two-phase forward--backward algorithm, of which backpropagation is a special case conventionally applied to neural networks. This is mainly due to the central role of scalar-valued objectives in ML, whose gradient with respect to a very large number of inputs can be evaluated exactly and efficiently with a single evaluation of the reverse mode. 

Reverse mode is a member of a larger family of AD algorithms that also includes the forward mode \citep{wengert1964simple}, which has the favorable characteristic of requiring only a single forward evaluation of a function (i.e., not involving any backpropagation) at a significantly lower computational cost. Crucially, forward and reverse modes of AD evaluate different quantities. Given a function $\bm{f}: \Re^n \to \Re^m$, forward mode evaluates the Jacobian--vector product $\bm{J}_{\bm{f}}\bm{v}$, where $\bm{J}_{\bm{f}} \in \Re^{m\times n}$ and $\bm{v} \in \Re^n$; and revese mode evaluates the vector--Jacobian product $\bm{v}^{\intercal} \bm{J}_{\bm{f}}$, where $\bm{v} \in \Re^m$. For the case of $f: \Re^n \to \Re$ (e.g., an objective function in ML), forward mode gives us $\nabla f \cdot \bm{v} \in \Re$, the directional derivative; and reverse mode gives us the full gradient $\nabla f \in \Re^n$.\footnote{We represent $\nabla f$ as a column vector.}

From the perspective of AD applied to ML, a ``holy grail'' is whether the practical usefulness of gradient descent can be achieved using only the forward mode, eliminating the need for backpropagation. This could potentially change the computational complexity of typical ML training pipelines, reduce the time and energy costs of training, influence ML hardware design, and even have implications regarding the biological plausibility of backpropagation in the brain \cite{bengio2015towards,lillicrap2020backpropagation}. In this work we present results that demonstrate stable gradient descent over a range of ML architectures using only forward mode AD.

\textbf{Contributions}
\begin{itemize}[leftmargin=3ex]
    \item We define the ``forward gradient'', an estimator of the gradient that we prove to be unbiased, based on forward mode AD without backpropagation.
    \item We implement a forward AD system from scratch in PyTorch, entirely independent of the reverse AD implementation already present in this library.
    \item We use forward gradients in stochastic gradient descent (SGD) optimization of a range of architectures, and show that a typical modern ML training pipeline can be constructed with only forward AD and no backpropagation.
    \item We compare the runtime and loss performance characteristics of forward gradients and backpropagation, and demonstrate speedups of up to twice as fast compared with backpropagation in some cases.
\end{itemize}

\textbf{A note on naming:} When naming the technique, it is tempting to adopt names like ``forward propagation'' or ``forwardprop'' to contrast it with backpropagation. We do not use this name as it is commonly used to refer to the forward evaluation phase of backpropagation,  distinct from forward AD. We observe that the simple name ``forward gradient'' is currently not used in ML, and it also captures the aspect that we are presenting a drop-in replacement for the gradient.

\section{Background}
\label{sec:background}

In order to introduce our method, we start by briefly reviewing the two main modes of automatic differentiation.

\subsection{Forward Mode AD}
\begin{center}
\vspace{-2mm}
\begin{tikzpicture}[scale=1]
    \node[circle,anchor=east] (theta) at  (0, 0) {$\bm{\theta}$};
    \node[circle,anchor=east] (v) at      (0, -3ex) {$\bm{v}$};
    \node[circle,anchor=west] (ftheta) at (20ex, 0) {$\bm{f}(\bm{\theta})$};
    \node[circle,anchor=west] (Jv) at (20ex, -3ex) {$\bm{J}_{\bm{f}}(\bm{\theta})\,\bm{v}$};
    \draw[->] (1ex, -1.5ex) -- (19ex, -1.5ex);
    \node[align=center] at (10ex,0ex) {Forward};
\end{tikzpicture}
\vspace{-8mm}
\end{center}

Given a function $\bm{f}:\Re^n \to \Re^m$ and the values $\bm{\theta}\in\Re^n$, $\bm{v}\in\Re^n$, forward mode AD computes $\bm{f}(\bm{\theta})$ and the \emph{Jacobian--vector product}\footnote{Popularized recently as a \texttt{jvp} operation in tensor frameworks such as JAX \citep{jax2018github}.} $\bm{J}_{\bm{f}}(\bm{\theta})\,\bm{v}$, where $\bm{J}_{\bm{f}}(\bm{\theta}) \in \Re^{m\times n}$ is the Jacobian matrix of all partial derivatives of $\bm{f}$ evaluated at $\bm{\theta}$, and $\bm{v}$ is a vector of perturbations.\footnote{Also called ``tangents''.} For the case of $f: \Re^n \to \Re$ the Jacobian--vector product corresponds to a directional derivative $\nabla f(\bm{\theta}) \cdot \bm{v}$, which is the projection of the gradient $\nabla f$ at $\bm{\theta}$ onto the direction vector $\bm{v}$, representing the rate of change along that direction.

It is important to note that the forward mode evaluates the function $\bm{f}$ and its Jacobian--vector product $\bm{J}_{\bm{f}}\bm{v}$ \emph{simultaneously in a single forward run.} Also note that $\bm{J}_{\bm{f}}\bm{v}$ is obtained without having to compute the Jacobian $\bm{J}_{\bm{f}}$, a feature sometimes referred to as a matrix-free computation.\footnote{The full Jacobian $\bm{J}$ can be computed with forward AD using $n$ forward evaluations of $\bm{J}\bm{e}_i,\;i=1,\dots n$ using standard basis vectors $\bm{e}_i$ so that each forward run gives us a single column of $\bm{J}$.}

\subsection{Reverse Mode AD}

\begin{center}
\begin{tikzpicture}[scale=1]
    \node[circle,anchor=east] (theta) at  (0, 0) {$\bm{\theta}$};
    \node[circle,anchor=east] (vJ) at     (0, -3ex) {$\bm{v}^{\intercal}\bm{J}_{\bm{f}}(\bm{\theta})$};
    \node[circle,anchor=west] (ftheta) at (20ex, 0) {$\bm{f}(\bm{\theta})$};
    \node[circle,anchor=west] (v) at (20ex, -3ex) {$\bm{v}$};
    \draw[->] (1ex, -0.1ex) -- (19ex, -0.1ex);
    \node[align=center] at (10ex,1.5ex) {Forward};
    \draw[<-] (1ex, -2.9ex) -- (19ex, -2.9ex);
    \node[align=center] at (10ex,-4.5ex) {Backward};
\end{tikzpicture}
\vspace{-5mm}
\end{center}

Given a function $\bm{f}:\Re^n \to \Re^m$ and the values $\bm{\theta}\in\Re^n$, $\bm{v}\in\Re^m$, reverse mode AD computes $\bm{f}(\bm{\theta})$ and the \emph{vector--Jacobian product}\footnote{Popularized recently as a \texttt{vjp} operation in tensor frameworks such as JAX \citep{jax2018github}.} $\bm{v}^{\intercal}\bm{J}_{\bm{f}}(\bm{\theta})$, where $\bm{J}_{\bm{f}}\in\Re^{m\times n}$ is the Jacobian matrix of all partial derivatives of $\bm{f}$ evaluated at $\bm{\theta}$, and $\bm{v}\in\Re^{m}$ is a vector of adjoints. For the case of $f: \Re^{n}\to\Re$ and $v=1$, reverse mode computes the gradient, i.e., the partial derivatives of $f$ w.r.t. all $n$ inputs $\nabla f(\bm{\theta})=\left[\frac{\partial f}{\partial \theta_1}, \dots , \frac{\partial f}{\partial \theta_n}\right]^{\intercal}$.

Note that $\bm{v}^{\intercal}\bm{J}_{\bm{f}}$ is computed in a single forward--backward evaluation, without having to compute the Jacobian $\bm{J}_{\bm{f}}$.\footnote{The full Jacobian $\bm{J}$ can be computed with reverse AD using $m$ evaluations of $\bm{e}_i^{\intercal}\bm{J},\;i=1,\dots m$ using standard basis vectors $\bm{e}_i$ so that each run gives us a single row of $\bm{J}$.}

\subsection{Runtime Cost}
\label{sec:runtime_cost}

The runtime costs of both modes of AD are bounded by a constant multiple of the time it takes to run the function $\bm{f}$ we are differentiating \citep{griewank2008evaluating}. Reverse mode has a higher cost than forward mode, because it involves data-flow reversal and it needs to keep a record (a ``tape'', stack, or graph) of the results of operations encountered in the forward pass, because these are needed in the evaluation of derivatives in the backward pass that follows. The memory and computation cost characteristics ultimately depend on the features implemented by the AD system such as exploiting sparsity \citep{gebremedhin2005color} or checkpointing \citep{siskind2018divide}. 

The cost can be analyzed by assuming computational complexities of elementary operations such as fetches, stores, additions, multiplications, and nonlinear operations \citep{griewank2008evaluating}. Denoting the time it takes to evaluate the original function $\bm{f}$ as $\textrm{runtime}(\bm{f})$, we can express the time taken by the forward and reverse modes as $R_f\times\textrm{runtime}(\bm{f})$ and $R_b\times\textrm{runtime}(\bm{f})$ respectively. In practice, $R_f$ is typically between 1 and 3, and $R_b$ is typically between 5 and 10 \citep{hascoet2014adjoints}, but these are highly program dependent. 

Note that in ML the original function corresponds to the execution of the ML code without any derivative computation or training, i.e., just evaluating a given model with input data.\footnote{Sometimes called ``inference'' by practitioners.} We will call this ``base runtime'' in this paper.

\section{Method}
\label{sec:method}

\subsection{Forward Gradients}
\label{sec:forwardgrad}
\begin{definition}
Given a function $f: \Re^n \to \Re$, we define the ``forward gradient'' $\bm{g}: \Re^n \to \Re^n$ as
\begin{align}
\bm{g}(\bm{\theta}) &= \left(\nabla f(\bm{\theta}) \cdot \bm{v} \right)\bm{v}\;,\label{eq:forwardgradient}
\end{align}
where $\bm{\theta}\in\Re^n$ is the point at which we are evaluating the gradient, $\bm{v}\in\Re^n$ is a perturbation vector taken as a multivariate random variable $\bm{v} \sim p(\bm{v})$\, such that $\bm{v}$'s scalar components $v_i$ are independent and have zero mean and unit variance for all $i$, and $\nabla f(\bm{\theta}) \cdot \bm{v} \in \Re$ is the directional derivative of $f$ at point $\bm{\theta}$ in direction $\bm{v}$.
\end{definition}

We first talk briefly about the intuition that led to this definition, before showing that $\bm{g}(\bm{\theta})$ is an unbiased estimator of the gradient $\nabla f(\bm{\theta})$ in Section~\ref{sec:proof}. 

As explained in Section~\ref{sec:background}, forward mode gives us the directional derivative $\nabla f(\bm{\theta}) \cdot \bm{v} = \sum_i \frac{\partial f}{\partial \theta_i} v_i$ directly, without having to compute $\nabla f$. Computing $\nabla f$ using only forward mode is possible by evaluating $f$ forward $n$ times with direction vectors taken as standard basis (or one-hot) vectors $\bm{e}_i \in \Re^n, i = 1 \dots n$, where $\bm{e}_i$ denotes a vector with a 1 in the $i$th coordinate and 0s elsewhere. This allows the evaluation of the sensitivity of $f$ w.r.t. each input $\frac{\partial f}{\partial \theta_i}$ separately, which when combined give us the gradient $\nabla f$.

In order to have any chance of runtime advantage over backpropagation, we need to work with a single run of the forward mode per optimization iteration, not $n$ runs.\footnote{This requirement can be relaxed depending on the problem setting and we would expect the gradient estimation to get better with more forward runs per optimization iteration.} In a single forward run, we can interpret the direction $\bm{v}$ as a weight vector in a weighted sum of sensitivities w.r.t. each input, that is $\sum_i \frac{\partial f}{\partial \theta_i}v_i$, albeit without the possibility of distinguishing the contribution of each $\theta_i$ in the final total. We therefore use the weight vector $\bm{v}$ to attribute the overall sensitivity back to each individual parameter $\theta_i$, proportional to the weight $v_i$ of each parameter $\theta_i$ (e.g., a parameter with a small weight had a small contribution and a large one had a large contribution in the total sensitivity).

In summary, each time the forward gradient is evaluated, we simply do the following:

\begin{itemize}[leftmargin=3ex]
\item Sample a random perturbation vector $\bm{v} \sim p(\bm{v})$, which has the same size with $f$'s argument $\bm{\theta}$.
\item Run $f$ via forward-mode AD, which evaluates $f(\bm{\theta})$ \emph{and} $\nabla f(\bm{\theta}) \cdot \bm{v}$ \emph{simultaneously} in the same single forward run, \textbf{without having to compute} $\bm{\nabla f}$ at all in the process. The directional derivative obtained, $\nabla f(\bm{\theta}) \cdot \bm{v}$, is a scalar, and is computed exactly by AD (not an approximation).
\item Multiply the scalar directional derivative $\nabla f(\bm{\theta}) \cdot \bm{v}$ with vector $\bm{v}$ and obtain $\bm{g}(\bm{\theta})$, the forward gradient.
\end{itemize}

Figure~\ref{fig:method} illustrates the process showing several evaluations of the forward gradient for the Beale function. We see how perturbations $v_k$ (orange) transform into forward gradients $(\nabla f \cdot v_k)v_k$ (blue) for $k\in [1, 5]$, sometimes reversing the sense to point towards the true gradient (red) while being constrained in orientation. The green arrow shows a Monte Carlo gradient estimate via averaged forward gradients, i.e., $\frac{1}{K} \sum_{k=1}^K (\nabla f \cdot v_k) v_k \approx \mathbb{E}[(\nabla f \cdot v) v]$.

\begin{figure}
\vskip 0.1in
\begin{center}
\centerline{\includegraphics[trim=0 0ex 0 0ex,clip,width=\columnwidth]{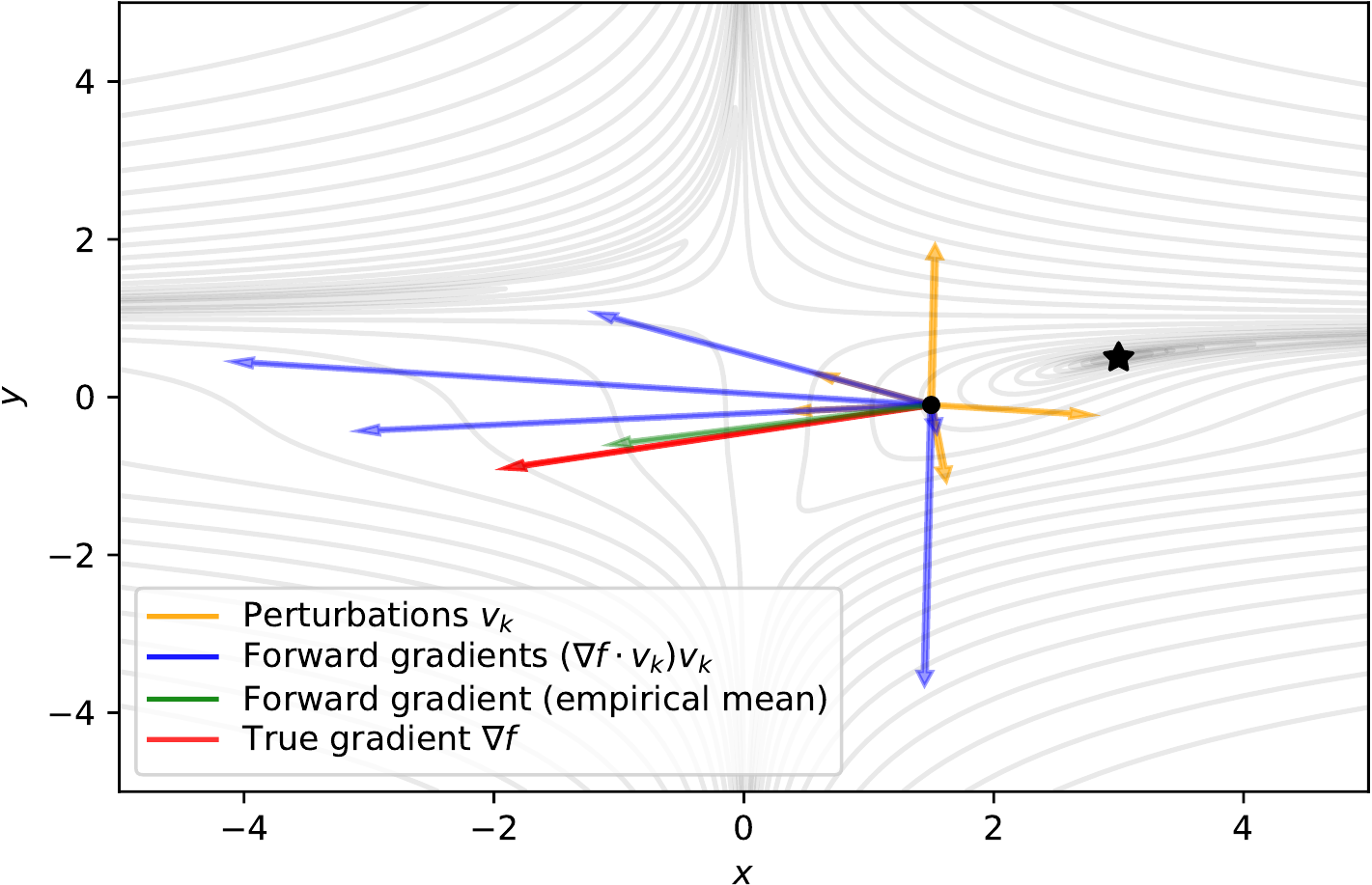}}
\caption{Five samples of forward gradient, the empirical mean of these five samples, and the true gradient for the Beale function (Section~\ref{sec:test_funcs}) at $x=1.5, y=-0.1$. Star marks the global minimum.}
\label{fig:method}
\end{center}
\vskip -0.2in
\end{figure}

\subsection{Proof of Unbiasedness}
\label{sec:proof}
\begin{theorem}
The forward gradient $\bm{g}(\bm{\theta})$ is an unbiased estimator of the gradient $\nabla f(\bm{\theta})$\;.
\end{theorem}

\begin{proof}
We start with the directional derivative of $f$ evaluated at $\bm{\theta}$ in direction $\bm{v}$ written out as follows
\begin{align}
d(\bm{\theta},\bm{v}) &= \nabla f(\bm{\theta}) \cdot \bm{v}
= \sum_{i} \frac{\partial f}{\partial \theta_i} v_i \nonumber\\
&= \frac{\partial f}{\partial \theta_1} v_1 + \frac{\partial f}{\partial \theta_2} v_2 + \dots + \frac{\partial f}{\partial \theta_n} v_n \;.
\end{align}

We then expand the forward gradient $\bm{g}$ in Eq.~(\ref{eq:forwardgradient}) as
\begin{align}
\bm{g}(\bm{\theta}) &= d(\bm{\theta}, \bm{v}) \bm{v} \nonumber\\
&= \left[\begin{array}{l@{\;}c@{\;}l@{\;}c@{\;\;}c@{\;\;}c@{\;}l}
    \frac{\partial f}{\partial \theta_1} v_1^2 &+ &\frac{\partial f}{\partial \theta_2} v_1 v_2 &+ &\cdots &+ &\frac{\partial f}{\partial \theta_n} v_1 v_n \\[1ex]
    \frac{\partial f}{\partial \theta_1} v_1 v_2 &+ &\frac{\partial f}{\partial \theta_2} v_2^2 &+ &\cdots &+ &\frac{\partial f}{\partial \theta_n} v_2 v_n\\
    &&&\vdots \\[0.8ex]
    \frac{\partial f}{\partial \theta_1} v_1 v_n &+ &\frac{\partial f}{\partial \theta_2} v_2 v_n &+ &\cdots &+ &\frac{\partial f}{\partial \theta_n} v_n^2
\end{array}\right]\nonumber
\end{align}
and note that the components of $\bm{g}$ have the following form
\begin{align}
g_i(\bm{\theta}) = \frac{\partial f}{\partial \theta_i} v_i^2 + \sum_{j \neq i} \frac{\partial f}{\partial \theta_j} v_i v_j \;.
\end{align}

The expected value of each component $g_i$ is
\begin{align}
\mathbb{E}\left[g_i(\bm{\theta})\right] &= \mathbb{E}\left[\frac{\partial f}{\partial \theta_i} v_i^2 + \sum_{j\neq i} \frac{\partial f}{\partial \theta_j} v_i v_j \right]\nonumber\\
&= \mathbb{E}\left[\frac{\partial f}{\partial \theta_i} v_i^2 \right] + \mathbb{E}\left[\sum_{j\neq i} \frac{\partial f}{\partial \theta_j} v_i v_j \right]\nonumber\\
&= \mathbb{E}\left[\frac{\partial f}{\partial \theta_i} v_i^2 \right] + \sum_{j\neq i} \mathbb{E}\left[\frac{\partial f}{\partial \theta_j} v_i v_j \right]\nonumber\\
&= \frac{\partial f}{\partial \theta_i}\mathbb{E}\left[v_i^2\right] + \sum_{j\neq i} \frac{\partial f}{\partial \theta_j}\mathbb{E}\left[v_i v_j\right] \label{eq:expected_g_i}
\end{align}

The first expected value in Eq.~(\ref{eq:expected_g_i}) is of a squared random variable and all expectations in the summation term are of two independent and identically distributed random variables multiplied.

\begin{lemma}
\label{lemma:squared}
The expected value of a random variable $v$ squared\; $\mathbb{E}[v^2]=1$ when\; $\mathbb{E}[v]=0$ and $\Var[v]=1$.
\end{lemma}
\begin{proof}
Variance is $\Var[v] = \mathbb{E}\left[(v - \mathbb{E}[v])^2\right] = \mathbb{E}\left[v^2\right] - \mathbb{E}[v]^2$. Rearranging and substituting $\mathbb{E}[v]=0$ and $\Var[v]=1$, we get $\mathbb{E}\left[v^2\right] = \mathbb{E}[v]^2 + \Var[v] = 0 + 1 = 1\;.$
\end{proof}

\begin{lemma}
\label{lemma:multiplied}
The expected value of two i.i.d. random variables multiplied $\mathbb{E}[v_i v_j]=0$ when $\mathbb{E}[v_i]=0$ or $\mathbb{E}[v_j]=0$.
\end{lemma}
\begin{proof}
For i.i.d. $v_i$ and $v_j$ the expected value $\mathbb{E}[v_i v_j] = \mathbb{E}[v_i] \mathbb{E}[v_j] = 0 $ when $\mathbb{E}[v_i]=0$ or $\mathbb{E}[v_j]=0$.
\end{proof}

Using Lemmas \ref{lemma:squared} and \ref{lemma:multiplied}, Eq.~(\ref{eq:expected_g_i}) reduces to
\begin{align}
\mathbb{E}\left[g_i (\bm{\theta})\right] &= \frac{\partial f}{\partial \theta_i}
\end{align}
and therefore
\begin{align}
\mathbb{E}\left[\bm{g}(\bm{\theta})\right] &= \nabla f(\bm{\theta})\;.
\end{align}
\end{proof}

\subsection{Forward Gradient Descent}

We construct a forward gradient descent (FGD) algorithm by replacing the gradient $\nabla f$ in standard GD with the forward gradient $\bm{g}$ (Algorithm~\ref{alg:fgd}). In practice we use a mini-batch stochastic version of this where $f_t$ changes per iteration as it depends on each mini-batch of data used during training. We note that the directional derivative $d_t$ in Algorithm~\ref{alg:fgd} can have positive or negative sign. When the sign is negative, the forward gradient $\bm{g}_t$  corresponds to backtracking from the direction of $\bm{v}_t$, or reversing the direction to point towards the true gradient in expectation. Figure~\ref{fig:method} shows two $v_k$ samples exemplifying this behavior.

In this paper we limit our scope to FGD to clearly study this fundamental algorithm and compare it to standard backpropagation, without confounding factors such as momentum or adaptive learning rate schemes. We believe that extensions of the method to other families of gradient-based optimization algorithms are possible.

\begin{algorithm}
 
\small
   \caption{Forward gradient descent (FGD)}
   \label{alg:fgd}
\begin{algorithmic}
   \REQUIRE $\eta$: learning rate
   \REQUIRE $f$: objective function
   \REQUIRE $\bm{\theta}_0$: initial parameter vector
   \STATE {$t \gets 0$\COMMENT{Initialize}}
   \WHILE{$\bm{\theta}_t$ not converged}
       \STATE $t \gets t + 1$
       \STATE $\bm{v}_t \sim \mathcal{N}(\bm{0}, \bm{I})$\COMMENT{Sample perturbation}

       \vspace{2mm}       
       \STATE \emph{\textbf{Note:} the following computes $f_t$ and $d_t$ simultaneously and without having to compute $\nabla f$ in the process}
       \STATE $f_t,\; d_t \gets f(\bm{\theta}_t),\; \nabla f(\bm{\theta}_t) \cdot \bm{v}$ \COMMENT{Forward AD (Section~\ref{sec:forwardgrad})}
       \vspace{2mm}

       \STATE $\bm{g}_t \gets \bm{v}_t d_t$\COMMENT{Forward gradient}
       \STATE {$\bm{\theta}_{t+1} \gets \bm{\theta}_t - \eta\,\bm{g}_t$\COMMENT{Parameter update}}
   \ENDWHILE
   \STATE {\bfseries return} $\bm{\theta}_t$
\end{algorithmic}
\end{algorithm}

\subsection{Choice of Direction Distribution}

As shown by the  proof in Section~\ref{sec:proof}, the multivariate distribution $p(\bm{v})$ from which direction vectors $\bm{v}$ are sampled must have two properties: (1) the components must be independent from each other (e.g., a diagonal Gaussian) and (2) the components must have zero mean and unit variance.

In our experiments we use the multivariate standard normal as the direction distribution $p(\bm{v})$ so that $\bm{v} \sim \mathcal{N}(\bm{0}, \bm{I})$, that is, $v_i \sim \mathcal{N}(0, 1)$ are independent for all $i$. We leave exploring other admissible distributions for future work.

\section{Related Work}

The idea of performing optimization by the use of random perturbations, thus avoiding adjoint computations, is the intuition behind a variety of approaches, including simulated annealing \citep{kirkpatrick1983optimization}, stochastic approximation \citep{spall1992multivariate}, stochastic convex optimization \citep{nesterov2017random,dvurechensky2021accelerated}, and correlation-based learning methods \citep{6313077}, which lend themselves to efficient hardware implementation \citep{NIPS1987_f033ab37}.
Our work here falls in the general class of so-called weight perturbation methods; see \citet[][\S{4.4}]{PEARLMUTTER94A} for an overview along with a description of a method for efficiently gathering second-order information during the perturbative process, which suggests that accelerated second-order variants of the present method may be feasible. Note that our method is novel in avoiding the truncation error of previous weight perturbation approaches by using AD rather than small but finite perturbations, thus completely avoiding the method of divided differences and its associated numeric issues.

In neural network literature, alternatives to backpropagation proposed include target propagation \cite{lecun1986learning,lecun1987phd,bengio2014auto,bengio2020deriving,meulemans2020theoretical}, a technique that propagates target values rather than gradients backwards between layers. For recurrent neural networks (RNNs), various approaches to the online credit assignment problem have features in common with forward mode AD \citep{pearlmutter1995gradient}. An early example is the real-time recurrent learning (RTRL) algorithm \citep{williams1989learning} which accumulates local sensitivities in an RNN during forward execution, in a manner similar to forward AD. A very recent example in the RTRL area is an anonymous submission we identified at the time of drafting this manuscript, where the authors are using directional derivatives to improve several gradient estimators, e.g., synthetic gradients \citep{jaderberg2017decoupled}, first-order meta-learning \citep{nichol2018first}, as applied to RNNs \citep{anonymous2022learning}.

Coordinate descent (CD) algorithms \citep{wright2015coordinate} have a structure where in each optimization iteration only a single component $\frac{\partial f}{\partial \theta_i}$ of the gradient $\nabla f$ is used compute an update. \citet{nesterov2012efficiency} provides an extension of CD called random coordinate descent (RCD), based on coordinate directional derivatives, where the directions are constrained to randomly chosen coordinate axes in the function's domain as opposed to arbitrary directions we use in our method. A recent use of RCD is by \citet{ding2021langevin} in Langevin Monte Carlo sampling, where the authors report no computational gain as the RCD needs to be run multiple times per iteration in order to achieve a preset error tolerance. The SEGA (SkEtched GrAdient) method by \citet{hanzely2018sega} is based on gradient estimation via random linear transformations of the gradient that is called a ``sketch'' computed using finite differences. Jacobian sketching by \citet{gower2018stochastic} is designed to provide good estimates of the Jacobian, in a manner similar to how quasi-Newton methods update Hessian estimates \citep{NEURIPS2018_d554f7bb}.

Lastly there are other, and more distantly related, approaches concerning gradient estimation such as synthetic gradients \citep{jaderberg2017decoupled}, motivated by a need to break the sequential forward--backward structure of backpropagation, and Monte Carlo gradient estimation \citep{mohamed2020monte}, where the gradient of an expectation of a function is computed with respect to the parameters defining the distribution that is integrated.

For a review of the origins of reverse mode AD and backpropagation, we refer the interested readers to \citet{schmidhuber2020who} and \citet{griewank2012who}.

\section{Experiments}
\vspace{-0.75mm}

\begin{figure*}[h]
\begin{center}
\centerline{\includegraphics[trim=0 0ex 0 4ex,clip,width=\textwidth]{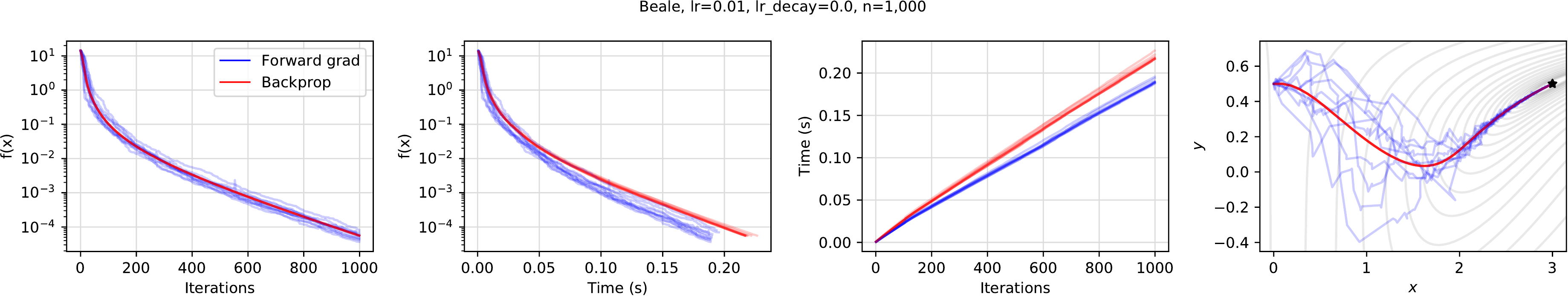}}
\centerline{\includegraphics[trim=0 0 0 4ex,clip,width=\textwidth]{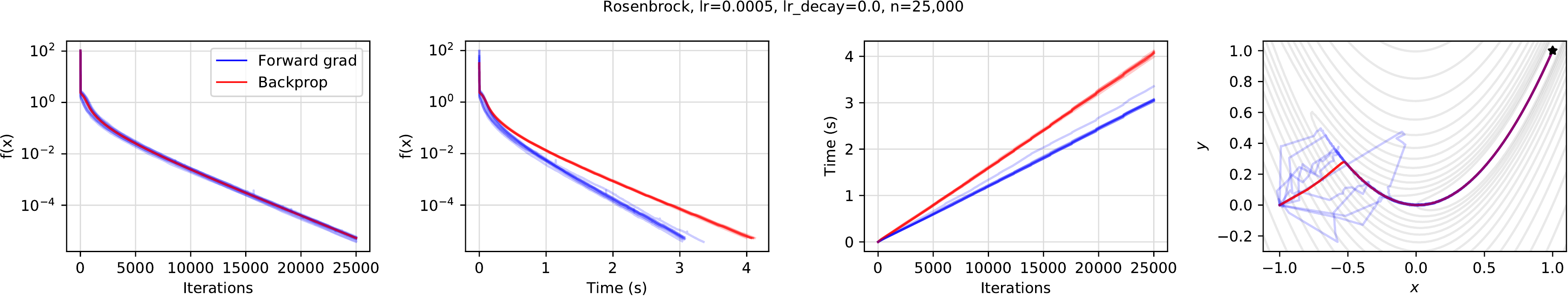}}
\vskip -0.1in
\vspace{-1mm}
\caption{Comparison of forward gradient and backpropagation in test functions, showing ten independent runs. Top row: Beale function, learning rate 0.01. Bottom row: Rosenbrock function. Learning rate $5\times 10^{-4}$. Rightmost column: Optimization trajectories in each function's domain, shown over contour plots of the functions. Star symbol marks the global minimum in the contour plots.}
\label{fig:test_functions}
\end{center} 
\vskip -0.1in
\end{figure*}

\begin{figure*}
\begin{center}
\centerline{\includegraphics[trim=0 0ex 0 4ex,clip,width=\textwidth]{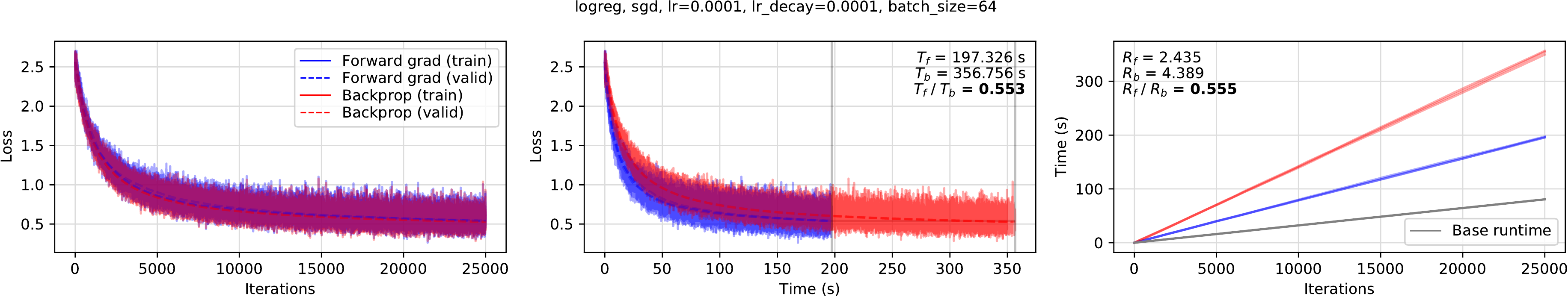}}
\vspace{-1mm}
\caption{Comparison of forward gradient and backpropagation in logistic regression, showing five independent runs. Learning rate $10^{-4}$.}
\label{fig:logreg}
\end{center}
\end{figure*}

We implement forward AD in PyTorch to perform the experiments (details given in Section~\ref{sec:implementation}). In all experiments, except in Section~\ref{sec:test_funcs}, we use learning rate decay with $\eta_{i} = \eta_0\, e^{-i\,k}$, where $\eta_i$ is the learning rate at iteration $i$, $\eta_0$ is the initial learning rate, and $k = 10^{-4}$. In all experiments we run forward gradients and backpropagation for an \emph{equal number of iterations.} We run the code with CUDA on a Nvidia Titan XP GPU and use a minibatch size of 64.

First we look at test functions for optimization, and compare the behavior of forward gradient and backpropagation in the $\Re^2$ space where we can plot and follow optimization trajectories. We then share results of experiments with training ML architectures of increasing complexity. We measured no practical difference in memory usage between the two methods (less than 0.1\% difference in each experiment).


\subsection{Optimization Trajectories of Test Functions}
\label{sec:test_funcs}

In Figure~\ref{fig:test_functions} we show the results of experiments with

\begin{itemize}[leftmargin=3ex]
\item the Beale function, $f(x,y) = (1.5-x+xy)^2 + (2.25 - x + xy^2)^2  + (2.625 -x + xy^3)^2$
\item and the Rosenbrock function, $f(x,y) = (a - x)^2 + b(y-x^2)^2$, where $a=1, b=100$.
\end{itemize}

Note that forward gradient and backpropagation have roughly similar time complexity in these cases, forward gradient being slightly faster per iteration. Crucially, we see that forward gradient steps behave the same way as backpropagation in expectation, as seen in loss per iteration (leftmost) and optimization trajectory (rightmost) plots.

\begin{figure*}[h]
\begin{center}
\includegraphics[trim=0 0ex 0 4ex,clip,width=\textwidth]{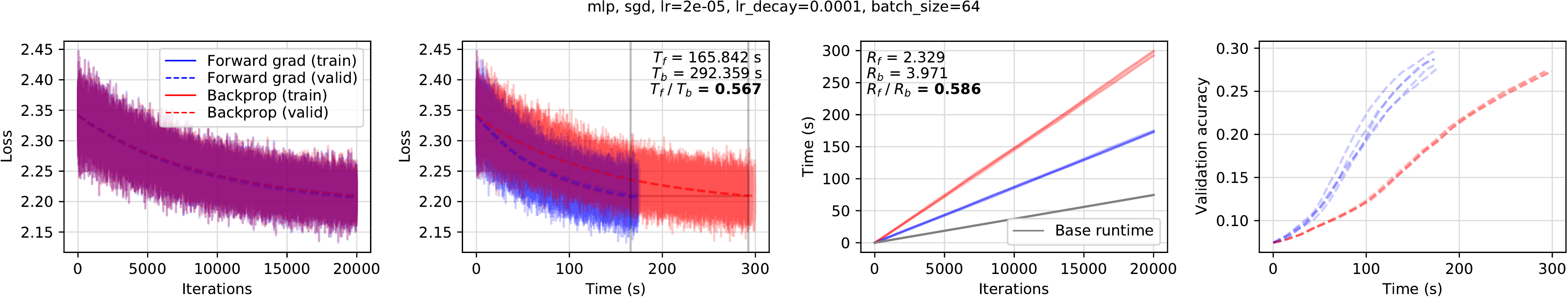}
\includegraphics[trim=0 0ex 0 4ex,clip,width=\textwidth]{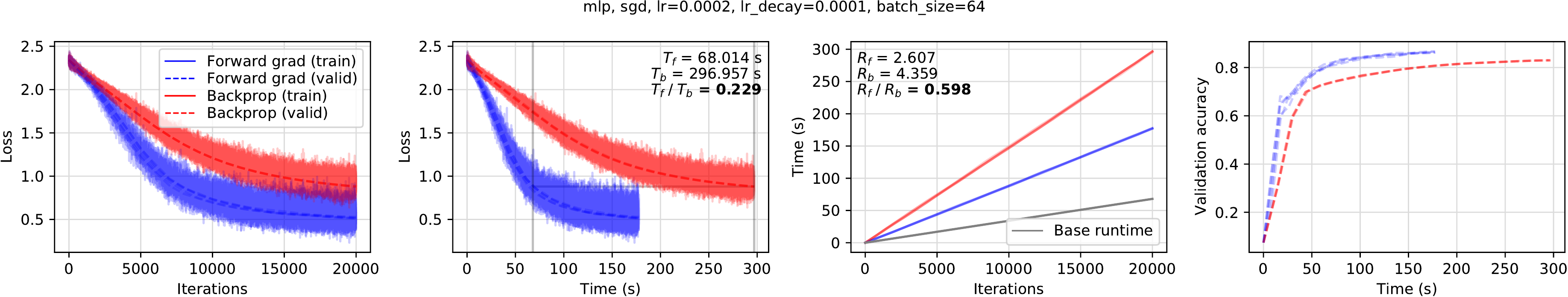}
\vspace{-4mm}
\caption{Comparison of forward gradient and backpropagation for the multi-layer NN, showing two learning rates. Top row: learning rate $2\times 10^{-5}$. Bottom row: learning rate $2\times 10^{-4}$. Showing five independent runs per experiment.}
\label{fig:mlp}
\end{center}
\end{figure*}

\begin{figure*}
\begin{center}
\includegraphics[trim=0 0ex 0 4ex,clip,width=\textwidth]{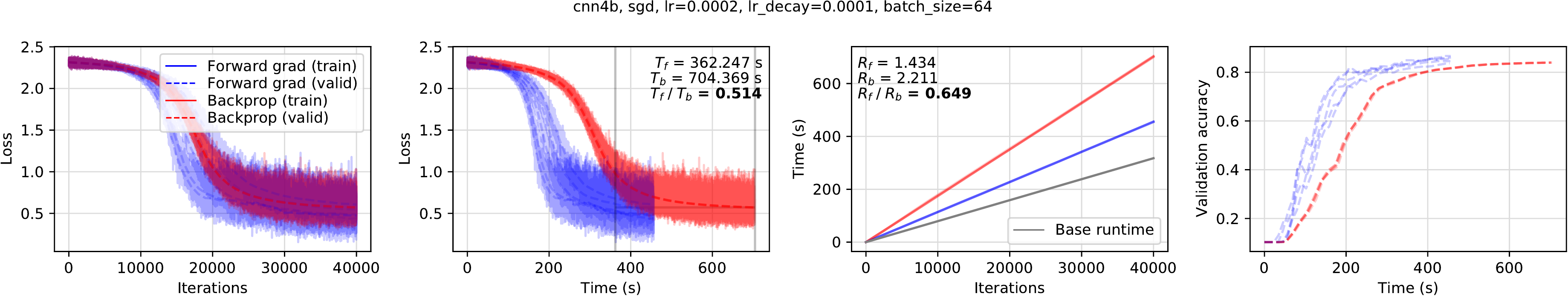}
\vspace{-5mm}
\caption{Comparison of forward gradient and backpropagation for the CNN. Learning rate $2\times 10^{-4}$. Showing five independent runs.}
\label{fig:cnn}
\end{center}
\vskip -0.1in
\end{figure*}

\subsection{Empirical Measures of Complexity}

In order to compare the two algorithms applied to ML problems in the rest of this section, we use several measures. 

For runtime comparison we use the $R_f$ and $R_b$ factors defined in Section~\ref{sec:runtime_cost}. In order to compute these factors, we measure $\textrm{runtime}(f)$ as the time it takes to run a given architecture with a sample minibatch of data and compute the loss, without performing any derivative computation and parameter update. Note that in the measurements of $R_f$ and $R_b$, the time taken by gradient descent (parameter updates) are included, in addition to the time spent in computing the derivatives. We also introduce the ratio $R_f / R_b$ as a measure of the runtime cost of the forward gradient relative to the cost of backpropagation in a given architecture.

In order to compare loss performance, we define $T_b$ as the time at which the lowest validation loss is achieved by backpropagation (averaged over runs). $T_f$ is the time the same validation loss is achieved by the forward gradient for the same architecture. The $T_f / T_b$ ratio gives us a measure of the time it takes for the forward mode to achieve the minimum validation loss relative to the time taken by backpropagation.

\subsection{Logistic Regression}

Figure~\ref{fig:logreg} gives the results of several runs of multinomial logistic regression for MNIST digit classification. We observe that the runtime cost of the forward gradient and backpropagation relative to the base runtime are $R_f = 2.435$ and $R_b = 4.389$, which are compatible with what one would expect from a typical AD system (Section~\ref{sec:runtime_cost}). The ratios $R_f / R_b = 0.555$ and $T_f / T_b = 0.553$ indicate that the forward gradient is roughly twice as fast as backpropagation in both runtime and loss performance. In this simple problem these ratios coincide as both techniques have nearly identical behavior in the loss per iteration space, meaning that the runtime benefit is reflected almost directly in the loss per time space. In more complex models in the following subsections we will see that the relative loss and runtime ratios can be different in practice.

\subsection{Multi-Layer Neural Networks}
\label{sec:mlp}

Figure~\ref{fig:mlp} shows two experiments with a multi-layer neural network (NN) for MNIST classification with different learning rates. The architecture we use has three fully-connected layers of size 1024, 1024, 10, with ReLU activation after the first two layers. In this model architecture, we observe the runtime costs of the forward gradient and backpropagation relative to the base runtime as $R_f = 2.468$ and $R_b = 4.165$, and the relative measure $R_f / R_b = 0.592$ on average. These are roughly the same with the logistic regression case.

The top row (learning rate $2\times 10^{-5}$) shows a result where forward gradient and backpropagation behave nearly identical in loss per iteration (leftmost plot), resulting in a $T_f / T_b$ ratio close to $R_f / R_b$. We show this result to communicate an example where the behavior is similar to the one we observed for logistic regression, where the loss per iteration behavior between the techniques are roughly the same and the runtime benefit is the main contributing factor in the loss per time behavior (second plot from the left).

Interestingly, in the second experiment (learning rate $2\times 10^{-4}$) we see that forward gradient achieves faster descent in the loss per iteration plot. We believe that this behavior is due to the different nature of stochasticity between the regular SGD (backpropagation) and the forward SGD algorithms, and we speculate that the noise introduced by forward gradients might be beneficial in exploring the loss surface. When we look at the loss per time plot, which also incorporates the favorable runtime of the forward mode, we see a loss performance metric $T_f / T_b$ value of 0.211, representing a case that is more than four times as fast as backpropagation in achieving the reference validation loss.

\subsection{Convolutional Neural Networks}

In Figure~\ref{fig:cnn} we show a comparison between the forward gradient and backpropagation for a convolutional neural network (CNN) for the same MNIST classification task. The CNN has four convolutional layers with $3\times3$ kernels and 64 channels, followed by two linear layers of sizes 1024 and 10. All convolutions and the first linear layer are followed by ReLU activation and there are two max-pooling layers with $2\times 2$ kernel after the second and fourth convolutions.

In this architecture we observe the best forward AD performance with respect to the base runtime, where the forward mode has $R_f=1.434$ representing an overhead of only 43\% on top of the base runtime. Backpropagation with $R_b=2.211$ is very close to the ideal case one can expect from a reverse AD system, taking roughly double the time. $R_f/R_b = 0.649$ represents a significant benefit for the forward AD runtime with respect to backpropagation. In loss space, we get a ratio $T_f/T_b = 0.514$ which shows that forward gradients are close to twice as fast as backpropagation in achieving the reference level of validation loss.

\subsection{Scalability}

The results in the previous subsections demonstrate that

\begin{itemize}[leftmargin=3ex]
\item training without backpropagation can feasibly work within a typical ML training pipeline and do so in a computationally competitive way, and
\item forward AD can even beat backpropagation in loss decrease per training time for the same choice of hyperparameters (learning rate and learning rate decay).
\end{itemize}

In order to investigate whether these results will scale to larger NNs with more layers, we measure runtime cost and memory usage as a function of NN size. In Figure~\ref{fig:scaling} we show the results for the MLP architecture (Section~\ref{sec:mlp}), where we run experiments with an increasing number of layers in the range $[1, 100]$. The linear layers are of size 1,024, with no bias. We use a mini-batch size 64 as before.

Looking at the cost relative to the base runtime, which also changes as a function of the number of layers, we see that backpropagation remains within $R_b \in [4, 5]$ and forward gradient remains within $R_f \in [3, 4]$ for a large proportion of the experiments. We also observe that forward gradients remain favorable for the whole range of layer sizes considered, with the $R_f/R_b$ ratio staying below 0.6 up to ten layers and going slightly over 0.8 at 100 layers. Importantly, there is virtually no difference in memory consumption between the two methods.

\begin{figure}[t]
\begin{center}
\includegraphics[trim=0 0ex 0 4ex,clip,width=\columnwidth]{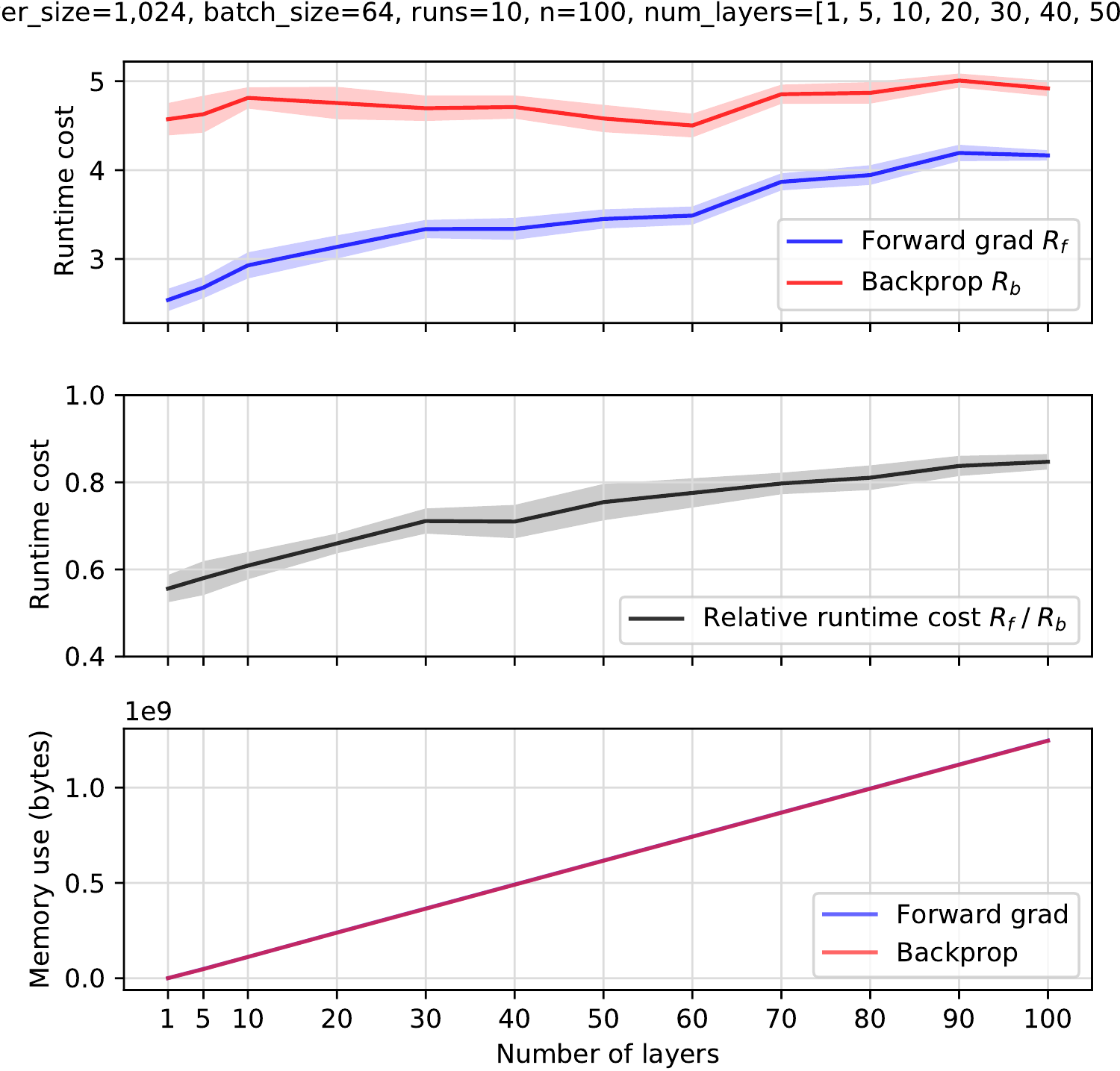}
\vspace{-5mm}
\caption{Comparison of how the runtime cost and memory usage of forward gradients and backpropagation scale as a function NN depth for the MLP architecture where each layer is of size 1024. Showing mean and standard deviation over ten independent runs.}
\label{fig:scaling}
\end{center}
\vskip -0.2in
\end{figure}

\section{Implementation}
\label{sec:implementation}

We implement a forward-mode AD system in Python and base this on PyTorch tensors in order to enable a fair comparison with a typical backpropagation pipeline in PyTorch, which is widely used by the ML community.\footnote{We also experimented with the forward mode implementation in JAX \citep{jax2018github} but decided to base our implementation on PyTorch due to its maturity and simplicity allowing us to perform a clear comparison.} We release our implementation publicly.\footnote{To be shared in the upcoming revision.}

Our forward-mode AD engine is implemented from scratch using operator overloading and non-differentiable PyTorch tensors (\texttt{requires\_grad=False}) as a building block. This means that our forward AD implementation does not use PyTorch's reverse-mode implementation (called ``autograd'') and computation graph. We produce the backpropagation results in experiments using PyTorch's existing reverse-mode code (\texttt{requires\_grad=True} and \texttt{.backward()}) as usual.

Note that empirical comparisons of the relative runtimes of forward- and reverse-mode AD are highly dependent on the implementation details in a given system and would show differences across different code bases. When implementing the forward mode of tensor operations common in ML (e.g., matrix multiplication, convolutions), we identified opportunities to the make forward AD operations even more efficient (e.g., stacking channels of primal and derivative parts of tensors in a convolution). Note that the implementation we use in this paper does not currently have these. We expect the forward gradient performance to improve even further as high-quality forward-mode implementations find their way into mainstream ML libraries and get tightly integrated into tensor code.

Another implementation approach that can enable a straightforward application of forward gradients to existing code can be based on the complex-step method \citep{martins2003complex}, a technique that can approximate directional derivatives with nothing but basic support for complex numbers.

\section{Conclusions}

We have shown that a typical ML training pipeline can be constructed without backpropagation, using only forward AD, while still being computationally competitive. We expect this contribution to find use in distributed ML training, which is outside the scope of this paper. Furthermore, the runtime results we obtained with our forward AD prototype in PyTorch are encouraging and we are cautiously optimistic that they might be the first step towards significantly decreasing the time taken to train ML architectures, or alternatively, enabling the training of more complex architectures with a given compute budget. We are excited to have the results confirmed and studied further by the research community.

The work presented here is the basis for several directions that we would like to follow. In particular, we are interested in working on gradient descent algorithms other than SGD, such as SGD with momentum, and adaptive learning rate algorithms such as Adam \citep{kingma2014adam}. In this paper we deliberately excluded these to focus on the most isolated and clear case of SGD, in order to establish the technique and a baseline. We are also interested in experimenting with other ML architectures. The components used in our experiments (i.e., linear and convolutional layers, pooling, ReLU nonlinearity) are representative of the building blocks of many current architectures in practice, and we expect the results to apply to these as well.

Lastly, in the longer term we are interested in seeing whether the forward gradient algorithm can contribute to the mathematical understanding of the biological learning mechanisms in the brain, as backpropagation has been historically viewed as biologically implausible as it requires a precise backward connectivity \citep{bengio2015towards,lillicrap2016random,lillicrap2020backpropagation}. In this context, one way to look at the role of the directional derivative in forward gradients is to interpret it as the feedback of a single global scalar quantity that is identical for all the computation nodes in the network.

We believe that forward AD has computational characteristics that are ripe for exploration by the ML community, and we expect that its addition to the conventional ML infrastructure will lead to major breakthroughs and new approaches.

\bibliography{main.bib}
\bibliographystyle{icml2022}



\end{document}